\documentclass{article}

\PassOptionsToPackage{numbers, square, sort, comma, compress}{natbib}

 \usepackage[preprint]{nips_2018}

\usepackage{tikz}
\usetikzlibrary{tikzmark}
\usepackage{rotating}
\usepackage{graphicx}
\usepackage{subcaption,booktabs}
\usepackage{tabularx}
\usepackage{ctable}
\usepackage{multirow}
\usepackage{float}
\usepackage[export]{adjustbox}
\usepackage{amsthm}
\usepackage{amsmath}
\usepackage{amssymb}
\usepackage{enumitem}
\usepackage{subfloat}
\usepackage{listings}

\newtheorem{thm}{Theorem}

\newtheorem{prop}{Proposition}

\newtheorem{defn}{Definition}

\newtheorem{assump}{Assumption}

\usepackage{balance}
\usepackage{colortbl}
\newcommand{\beginsupplement}{%
        \setcounter{section}{0}
        \renewcommand{\thesection}{S\arabic{section}}%
        \setcounter{algorithm}{0}
     \renewcommand{\thealgorithm}{S\arabic{algorithm}}
        \setcounter{table}{0}
        \renewcommand{\thetable}{S\arabic{table}}%
        \setcounter{figure}{0}
        \renewcommand{\thefigure}{S\arabic{figure}}%
        \setcounter{thm}{0}
     }

\usepackage{times}
\usepackage{graphicx} 
\usepackage{diagbox}

\DeclareMathOperator{\E}{\mathbb{E}}

\usepackage{algorithm}
\usepackage{algorithmic}

\usepackage[draft]{hyperref}

\newcommand{\ie}{\textit{i}.\textit{e}., }

\newcommand*\samethanks[1][\value{footnote}]{\footnotemark[#1]}

\title{Domain Adaptation Using Adversarial Learning for Autonomous Navigation}

\author{
  Jaeyoon Yoo\thanks{These authors contributed equally to this work.} \\
  Electrical and Computer Engineering\\
  Seoul National University\\
  Seoul 08826, Korea  \\
  \texttt{yjy765@snu.ac.kr} \\
  \And
  Yongjun Hong\samethanks \\
  Electrical and Computer Engineering \\
  Seoul National University\\
  Seoul 08826, Korea \\
  \texttt{yjhong@snu.ac.kr} \\
  \AND
  Yungkyun Noh \\
  Mechanical and Aerospace Engineering \\
  Seoul National University\\
  Seoul 08826, Korea \\
  \texttt{yungkyun.noh@gmail.com} \\
  \And
  Sungroh Yoon\thanks{To whom correspondence should be addressed.} \\
  Electrical and Computer Engineering \\
  Seoul National University \\
  Seoul 08826, Korea \\
  \texttt{sryoon@snu.ac.kr} \\
}

\begin{document}

\maketitle

\begin{abstract}
  Autonomous navigation has become an increasingly popular machine learning application. Recent advances in deep learning have also resulted in great improvements to autonomous navigation. However, prior outdoor autonomous navigation depends on various expensive sensors or large amounts of real labeled data which is difficult to acquire and sometimes erroneous. The objective of this study is to train an autonomous navigation model that uses a simulator (instead of real labeled data) and an inexpensive monocular camera. In order to exploit the simulator satisfactorily, our proposed method is based on domain adaptation with adversarial learning. Specifically, we propose our model with 1) a dilated residual block in the generator, 2) cycle loss, and 3) style loss to improve the adversarial learning performance for satisfactory domain adaptation. In addition, we perform a theoretical analysis that supports the justification of our proposed method. We present empirical results of navigation in outdoor courses with various intersections using a commercial radio controlled car. We observe that our proposed method allows us to learn a favorable navigation model by generating images with realistic textures. To the best of our knowledge, this is the first work to apply domain adaptation with adversarial learning to autonomous navigation in real outdoor environments. Our proposed method can also be applied to precise image generation or other robotic tasks.
\end{abstract}

\section{Introduction}

Autonomous navigation for vehicles has attracted great attention recently in the machine learning field. Recent advances in deep learning~\cite{sunwoo1} have garnered significant achievements~\cite{carNvidia}. However, the majority of prior studies required large labeled datasets, which often thwarted further progress because of the expense of human labeling.

In this paper, we propose a data-driven method for autonomous navigation that employs domain adaptation via adversarial learning networks \cite{DANN,ARDA}, and takes only a monocular raw image as an input. Our approach has two major advantages over the previous works. 

Firstly, our method alleviates the labeling issue through domain adaptation by adversarially generating realistic images from simulated images. Domain adaptation refers to the adaptation of one domain to another such that the data from the former domain can be used to train a model for the tasks in the latter domain \cite{transfer}. We apply domain adaption to a simulator in order to train a navigational model that works reasonably well in a real environment without any real labeled data. Simulated environments do not suffer from labeling issues \cite{Berkeley}. 
These days simulated images are often highly realistic, which attracts us to use a simulator \cite{jaehong1} to generate training images.

Secondly, our method requires only a monocular camera for sensing. Several existing approaches to autonomous navigation depend on various sensors, including global positioning system (GPS) sensors and depth sensors \cite{carsensor1,carsensor2,carsensor3,gps1}. Such approaches are effective but often raise issues. The additional cost and complexity are the main issues that are encountered with the use of multiple sensors. 
They may increase the actuator burden and have their own limitations.
For example, a GPS sensor does not usually operate well in an indoor environment or in a forest. 
As compared with other sensors, a monocular camera is light and cost-effective. 

Generating realistic images while preserving the content of simulated images is crucial for realizing a successful navigational model on real world data. We demonstrate that successful domain adaptation can be achieved using our proposed method, which incorporates a dilated residual block in the generator, cycle loss, style loss, and other techniques such as patch discriminator and soft label. We have observed that our method greatly helps to make simulated images realistic by changing not only their color but also their texture (of elements such as grass and trees) while maintaining the road and buildings, which aids in the learning of a favorable navigational model.

We evaluate the performance of our method by performing a navigational task using an RC car in the real world with various intersections. We obtained results by training a deep learning model using a myriad of simulator images with auto-generated labels and a few real images without any label. An additional experiment with a local trail environment, which can be applied to Unmanned Aerial Vehicle (UAV) road following, is referred to in supplementary material S5.

We summarize our contributions as follows:
\begin{itemize}[leftmargin=*]
\item We successfully performed autonomous navigation of outdoor roads with various intersections in an unsupervised manner. 
\item We found an optimal pipeline for performing domain adaptation successfully through the use of cycle loss, style loss, dilated convolution, soft label and other techniques.
\item We showed how our method can outperform existing methods theoretically even though it does not require real labels.
\end{itemize}

\section{Related Work} \label{related}
\subsection{Learning for Autonomous Navigation}
The majority of the learning based autonomous navigation models are trained using supervised learning. \citet{carNvidia} realized autonomous highway driving and some studies improved them by changing learning method \cite{Dagger,PLATO,carDagger, Rccar1, carend} or learning driving state associated with driving actions \cite{deepdriving,guisti,NVIDIA}. However, these methods require real labeled data which is usually expensive. 

Some studies exploited a simulator to train an autonomous navigation model for some advantages: 1) Ability to measure any physical property and change model configuration at a very low cost. 2) No safety issues. However, it is difficult to use the model learned from the simulator directly in reality. To address this issue, \citet{Depthnet} and \citet{RCcar_transfer} exploited input data, which is less variant across the two environments instead of raw images. However, their methods depend on additional sensors or complex subsystems. \citet{CAD2RL} suggested the use of randomization with widely varying features. Although they demonstrated the performance of  their navigation model in an indoor environment, it may be challenging to apply it to an outdoor environment owing to the high complexity involved \cite{RCcar_transfer}. \citet{drivingGAN} and \citet{fuck} are works that are the most similar to ours and use generative adversarial networks (GANs) \cite{GAN,lee2017seqgan,yoo2017energy,gansurvey,hwang2017disease} to adapt the simulator to reality; however, they did not demonstrate actual driving in an outdoor environment. Moreover, details to improve domain adaptation performance in their works are different from ours.

\subsection{Domain Adaptation}
Domain adaptation is a method that involves the adaptation of data from one domain (\ie a source domain) into another domain (\ie a target domain) while the classification task performance is preserved in a target domain \cite{transfer}. By doing so, we can train a classifier to work satisfactorily for the target domain data with plentiful labeled source domain data and unlabeled target domain data. 


The important part of domain adaptation is the reduction of the differences in the two domains' distributions, which is called a domain discrepancy \cite{domaindisc}. One method to achieve this is to construct a common representation space wherein the two distributions are projected indistinguishably from each other \cite{DANN,domaindisc,ARDA,transfer_survey2}, and then train a task model with the projected source domain data. Intuitively, as the two projected distributions approach each other, the domain discrepancy becomes smaller and the task model performs satisfactorily in the target domain \cite{DANN}. 
In order to find a common representation space, several methods such as using re-weighting \cite{transfer1,transfer2} and finding a feature space/subspace transformation \cite{transfer3,transfer4,transfer5} have been proposed. Recently, some trials have been performed, in which a GAN framework was used to make the feature representations indistinguishable between the two domains while maintaining the ability of the task classifier \cite{DANN,unsupervised,ARDA,CYCADA}. 

Domain adaptation has been applied to various tasks. \citet{mav} learned transferable policies for UAV navigation using the maximum mean discrepancy \cite{MMD, mmd1,mmdGAN} of a predefined kernel.
There are also several studies that employed domain adaptation to exploit synthesized data for eye-gaze estimation, pose estimation and other robotics tasks \cite{Berkeley,Apple,DAex1,DAex2,DAex3}.

\section{Method}
\begin{figure}[t]
\centering
\includegraphics[width=\linewidth, height=4cm]{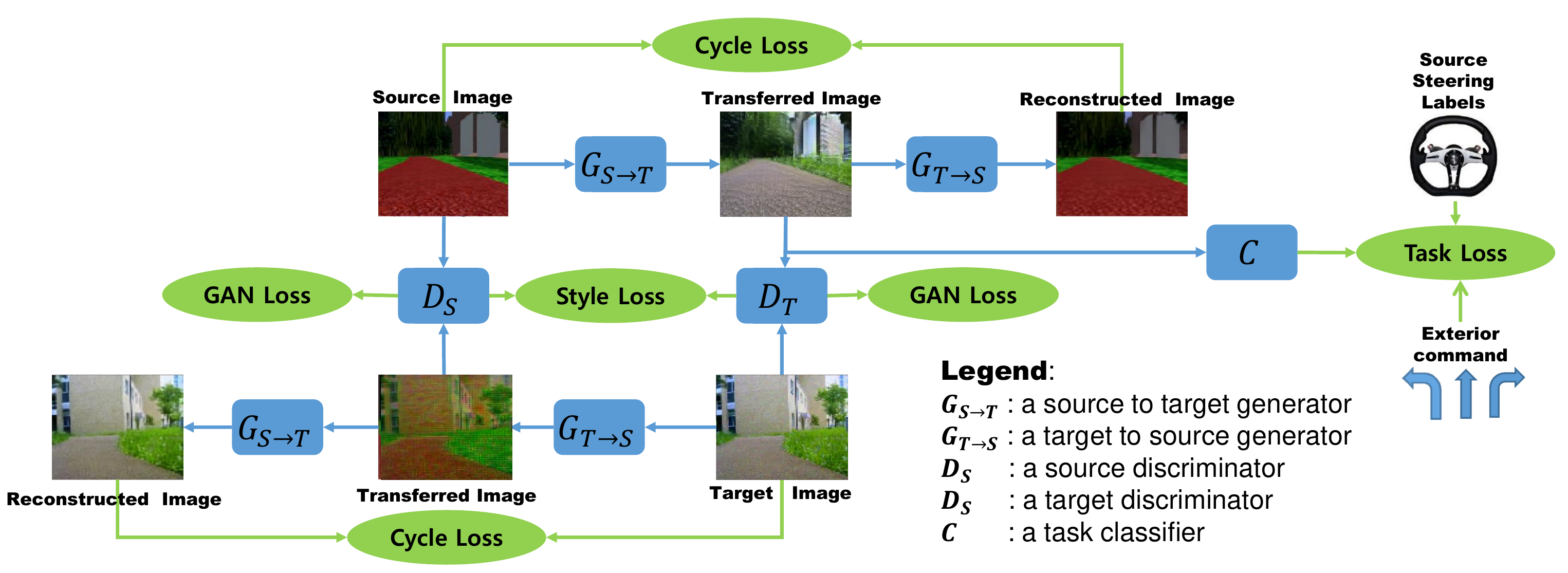}
\caption{Overview of our proposed method.}
\label{fig:overview}
\end{figure}

Our objective is to train a navigation system performing a navigation task only using simulator images, real images, and correct steering labels for corresponding simulator images. To address the difference between simulator images and real images, and the absence of labels for real images, we adopt domain adaptation with an adversarial learning framework. By transferring a simulator image (\ie source domain) into a transferred image (\ie transferred domain) to make it look like a real image (\ie target domain) while maintaining its characteristics with adversarial learning, the domain discrepancy is reduced and an accurate navigational model can be achieved in the target domain. In this section, we explain our proposed method, which can be used to significantly reduce domain discrepancy and develop an excellent navigational model.

\subsection{Model Components}\label{model_details}




An overview of our model is illustrated in Fig.~\ref{fig:overview}. Our model contains 
two generators $G_{T\rightarrow S}$ and $G_{S\rightarrow T}$, two discriminators $D_{T}$ and $D_{S}$, and a task classifier $C$ with the trainable parameters $\theta_{g^T}, \theta_{g^S}, \theta_{d^T}, \theta_{d^S}, \theta_{c}$ respectively. The subscript $S$ and $T$ indicate the domains of the input image (source/target). The generator converts the input image from one domain to fit to the other domain. The discriminator takes the transferred images and real images and distinguishes between them. The classifier, which has branches at the last layer \cite{Rccar1} to perform a conditional navigation task, takes the transferred image and the ground-truth label of the source image, and then learns to classify.

We design the generator with residual blocks \cite{he2016deep}. We replace some down-sampling layers with the dilated convolution \cite{drn} not to lose spatial information without reducing receptive fields, which is crucial for scenery understanding. We call this architecture as dilated residual network (DRN) \cite{drn} in the generator. We adopt the dilated convolution in the classifier for the same purpose.

In addition, we attempt to apply UNET \cite{unet} to our generator, which is a dominant model for pixel-to-pixel image generation tasks such as image segmentation \cite{unet,cciccek20163d} and style transfer \cite{pix2pix}. By including skip connections, UNET can utilize the low level features of the input image. We compare the quality of the transferred images in Sec.\ref{compare}.

We adopt the patchGAN in \cite{pix2pix} for the discriminator to generate high-resolution image. The patchGAN has multiple discriminators for local patches of the image to capture local high-frequency parts of the image. We empirically observed that the use of patchGAN increases a resolution of the transferred image compared with a single discriminator as the size of the transferred image increases. Details regarding its implementation are described in supplementary material S1.


\subsection{Training Loss}\label{loss}

Our objective function contains three loss terms: adversarial loss, style loss, and task loss.



1) \textbf{Adversarial loss.} We adopt the loss terms of the Fisher GAN \cite{FISHERGAN} inspired from the Fisher linear discriminant analysis \cite{flda}. The Fisher GAN belongs to the integral probability metrics (IPMs) framework \cite{ipm}, which is known to be a strongly and consistently convergent metric family \cite{ipmfdiv,gansurvey}. 
Moreover, the Fisher GAN is more computationally beneficial than other IPMs such as those in \cite{wasserstein, wgp} as it does not require the calculation of second derivatives for estimating the distance \cite{FISHERGAN}
while maintaining a stable training property.

In addition, we adopt L1 cyclic consistency loss between the input and reconstructed images based on \cite{cyclegan,disco,CYCADA}. We anticipate that the use of this loss can aid in preserving the content of an image, especially the road, which is important for training an accurate navigational model.

The adversarial loss for the target input image is defined as follows:
\begin{align}
L_A^T = \Phi(\theta_{d}^T, \theta_{g}^T) + \lambda(1-\Omega(\theta_d^T, \theta_g^T))-\frac{\rho}{2}(\Omega(\theta_d^T, \theta_g^T)-1)^2 +\E_{x\sim P_\mathrm{T}} \Vert G_{S\rightarrow T}(G_{T\rightarrow S}(x)) - x \Vert_1
\label{fisherfirst} 
\end{align}
where $\lambda$ is the Lagrange multiplier, $\rho>0$ is the quadratic penalty weight coefficient \cite{FISHERGAN}, and $P_\mathrm{S}$ and $P_\mathrm{T}$ are the source and target image distributions, respectively. $\Phi(\theta_d, \theta_g)$ and $\Omega(\theta_d, \theta_g)$ are equations derived from augmented Lagrangian \cite{augmented,FISHERGAN} and defined as follows:
\begin{align}
\Phi(\theta_d^T, \theta_g^T) &= \E_{x\sim P_\mathrm{T}}D_{T}(x)-\E_{x \sim P_\mathrm{S}}D_{T}(G_{S\rightarrow T}(x)) \label{fishersec} \\
\Omega(\theta_d^T, \theta_g^T) &= \frac{1}{2}[\E_{x\sim P_\mathrm{T}}D_{T}^2(x)+\E_{x\sim P_\mathrm{S}}D_{T}^2(G_{S\rightarrow T}(x))] \label{fisherlast}
\end{align}
The adversarial loss for the source input image $L_A^S$ can be derived in a similar manner.

2) \textbf{Style loss.} As the source and target images differ in their style and texture, we adopt a style loss in order to allow the transferred image to have a similar style to the target images. We are motivated by style transfer which is, when given two images, to generate an image containing the content of one image and the style of the other image. To obtain the representation of a style, we use the gram matrix which indirectly represents feature distribution as in \cite{style2,style3,style4}. Instead of using a pretrained network, we utilize the middle layers' features of the discriminator to obtain the gram matrix, which works surprisingly well, as shown in Sec.\ref{experiment}.



To match the styles of the real and transferred images, we attempt to reduce the squared Frobenius norm of the gram matrix of the real image, $A^l$ and that of the fake image, $B^l$ in the $l_{th}$ layers of the discriminator. The style loss is defined as follows:
\begin{align}
L_S = \sum_{l=1}^L \frac{1}{(n^l)^2 (c^l)^2} \sum_{i,j} (A_{ij}^l-B_{ij}^l)^2 \label{style loss}
\end{align}
where $n^l$ and $c^l$ represent the product of the batch size and the area of the feature map and the number of the feature map, respectively, in the layer $l$. $L$ is the number of layers in the discriminator.


3) \textbf{Task loss.} We divide the steering command into five intervals to assign class labels and train the classifier using typical cross entropy loss. It should be noted that instead of strict labels for each class, we use soft labels that assign a small portion of the label to adjacent classes \cite{salimans2016improved}. It provides the classifier with the relation between the adjacent classes and prevents the classifier from producing extremely sharp results. \citet{NVIDIA} showed that an overly sharp output may hinder navigation. As compared with the use of the negative entropy and swap penalty terms proposed in \cite{NVIDIA}, our method is simpler and produces the same desired result.

\subsection{Theoretical Analysis} \label{md}

We address the lack of labels in the target domain by making the source images resemble the target images, and giving them the labels of the source images. Intuitively, if a classifier works well for the transferred images and the transferred images are similar to the target images, then the classifier may fit to the target images. In this section, we argue that our intuition can be explained mathematically by presenting the upper bound of the trained model's performance under a mild assumption. 

We refer to the image and label spaces as $\mathcal{X}$ and $\mathcal{Y}$, respectively. For an image $x \in \mathcal{X}$ and a label $y\in \mathcal{Y}$, $P_\mathrm{T}(x,y)$ and $P_\mathrm{F}(x,y)$ defined on $\mathcal{X} \times \mathcal{Y}$ denote the joint probability on the target and the transferred domains, respectively. $P_\mathrm{T}(x)$ and $P_\mathrm{T}(y|x)$ are the marginal distribution and conditional probability for the joint probability, $P_\mathrm{T}(x,y)$, respectively. The same is true for $P_\mathrm{F}(x)$ and $P_\mathrm{F}(y|x)$. Below is the list of a definition and assumption used in our analysis.

\begin{defn}\label{def1}
An error rate $e(x,y)$ denotes the inaccuracy of the trained classifier for an image $x$ and its label $y$. The range of the error rate is [0,1]. Then an average error in the target domain is the average error rate on the target domain, \ie
\begin{equation} \label{eq:errorgdd}
E_\mathrm{T} = \int_{\mathcal{X} \times \mathcal{Y}} e(x,y)P_\mathrm{T}(x,y)dxdy.
\end{equation}
An average error in the transferred domain, $E_\mathrm{F}$, is defined in the same manner. It should be noted that we use integral $\int\cdot$ for the discrete space $\mathcal{Y}$ for convenience.
\end{defn}

We can interpret $E_\mathrm{T}$ as the measure of the trained model's inaccuracy on the target domain. The same is true for $E_\mathrm{F}$.

\begin{assump}\label{assump1}
The class information has been transferred correctly, \ie $P_\mathrm{F}(y|x) \approx P_\mathrm{T}(y|x)$ in $supp(P_\mathrm{T})\bigcup supp(P_\mathrm{F}) \subset \mathcal{X} \times \mathcal{Y}$, where $supp(\cdot)$ indicates the support of a probability distribution.
\end{assump}



Assumption~\ref{assump1} can be supported with practical implementations such as the cyclic loss in Sec.\ref{loss}. 
We can then obtain the upper bound of the error of the classifier on the target domain as follows:

\begin{thm}\label{thm1}
$E_\mathrm{T} \leq 4 \sqrt{X_2^2(P_\mathrm{T}(x),P_\mathrm{F}(x))} + E_\mathrm{F}$ where $X_2^2$ is the Chi-squared distance.
\end{thm}
\begin{proof}
See supplementary material S2.
\end{proof}

According to Theorem 1.1 in \citet{FISHERGAN}, the Fisher IPM is equivalent to the Chi-squared distance, $X_2^2(P_\mathrm{T}(x),P_\mathrm{F}(x))$. It should, thus, be small if we perform the Fisher GAN training successfully. $E_\mathrm{F}$ should also be small if we train the classifier correctly. Along with our training result, this analysis explains why our empirical results in Sec.\ref{experiment} show excellent performance.

\section{Experiments} \label{experiment}

%

We evaluate our model using two criteria. First, to check how well our model performs domain adaptation, we compare the transferred images qualitatively. Second, in order to observe the navigational performance in the real world, we test how well the vehicle finishes courses with various intersections. We adopt the problem setup used in ~\cite{Rccar1}. One of the three navigational commands (turn left/go straight/turn right) is given to the vehicle at each time. The navigational model should still determine how to pass a winding or curved road under the `go straight' command and when to/how to turn under the `turn left/right' command.

\subsection{Setup}
We used the ROBOTIS Turtlebot3, a commercial RC car robot for our experiment. We mounted three webcams facing the left, forward and right directions. 
We gathered approximately 5000 pairs of images and steering labels in a simulator. To alleviate the covariate shift~\cite{Dagger} and improve the generalization, we added noise while collecting the data and performed data augmentation by perturbing contrast and saturation of the images randomly as in ~\cite{Rccar1}. Aside from the simulator images for the source images, we gathered approximately 5,000 pairs of images for the target images. A few examples of these images are shown in Fig.\ref{env}.

 \begin{figure}[t]
 \centering
   \begin{minipage}[t]{.24\textwidth}
   \includegraphics[width=\linewidth, height=3.5cm]{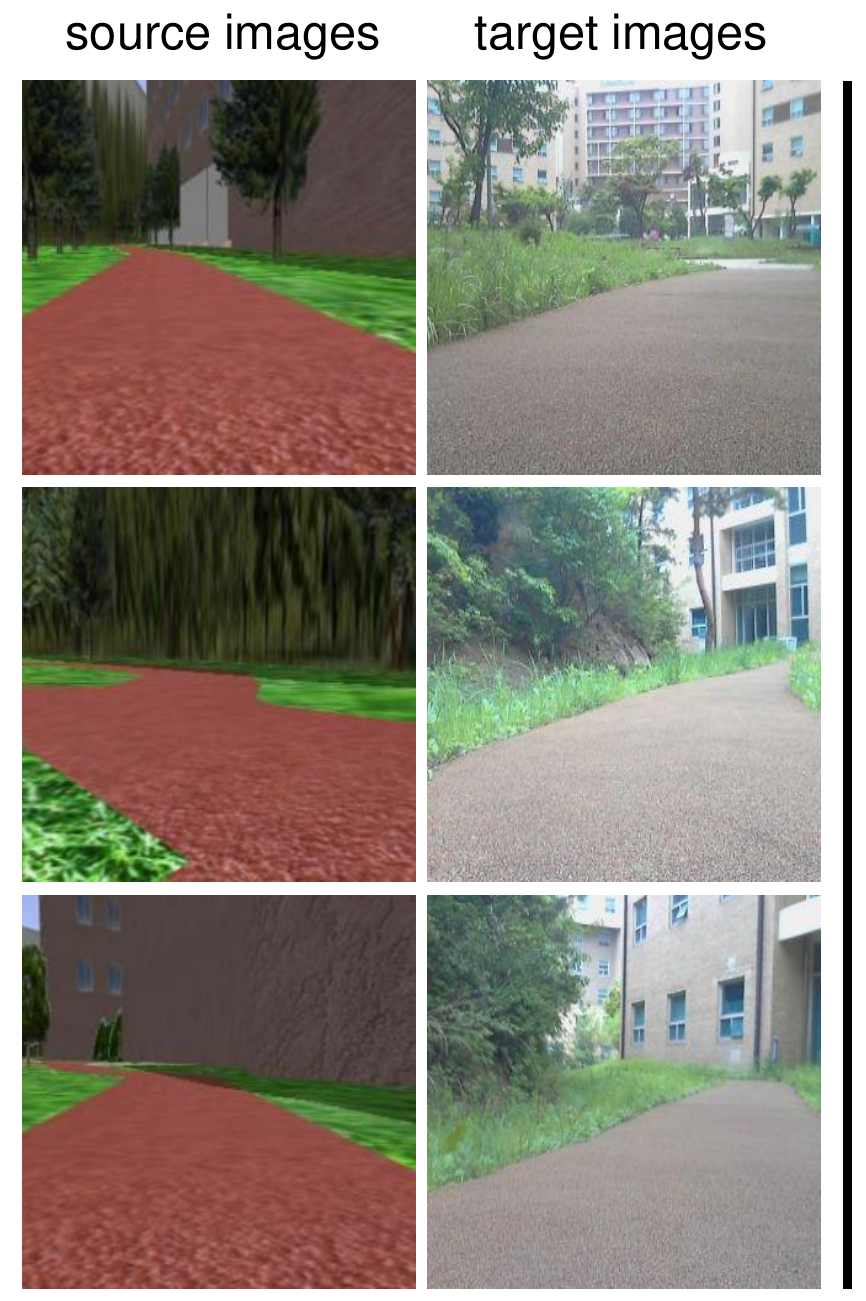}
   \caption{Source and target images}
   \label{env}  
   \end{minipage}
   \hspace{0.2cm}
   \begin{minipage}[t]{.72\textwidth}
   \centering
   \includegraphics[width=\linewidth, height=3.5cm]{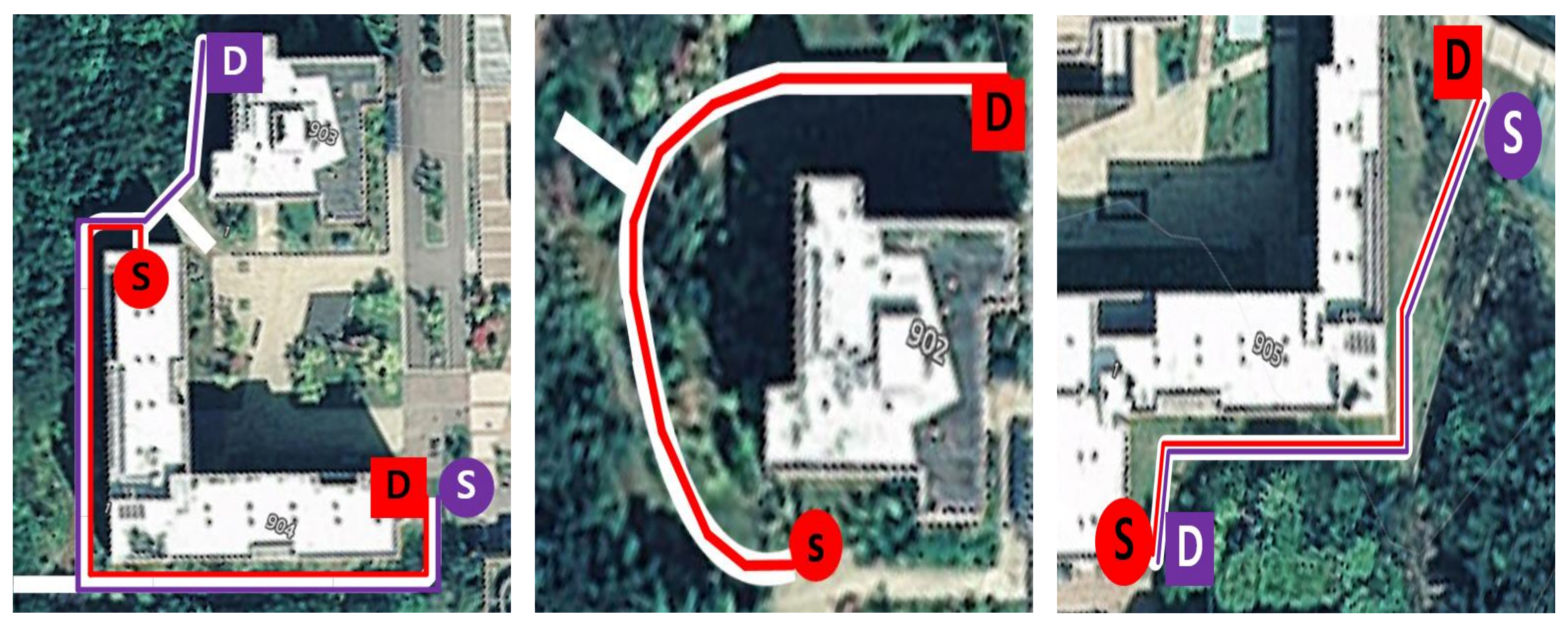}
   \caption{Course photograph. (left) Course 1-red line and course 2-purple line (center) Course 3 (right) Course 4-red and purple line. S is a start point and D is a destination.}
   \label{course}
   \end{minipage}
 \end{figure} 



\subsection{Comparison on Transferred Images} \label{compare}
The core of our method is how to reduce the domain discrepancy effectively through adversarial learning. In this section, we argue that using DRN for the generator, cycle loss, and style loss together is essential for our objective by comparing the quality of the transferred images. 

\begin{figure}[t]
\vskip 0.1in
\begin{center}
\centerline{\includegraphics[width=\columnwidth,height=9cm]{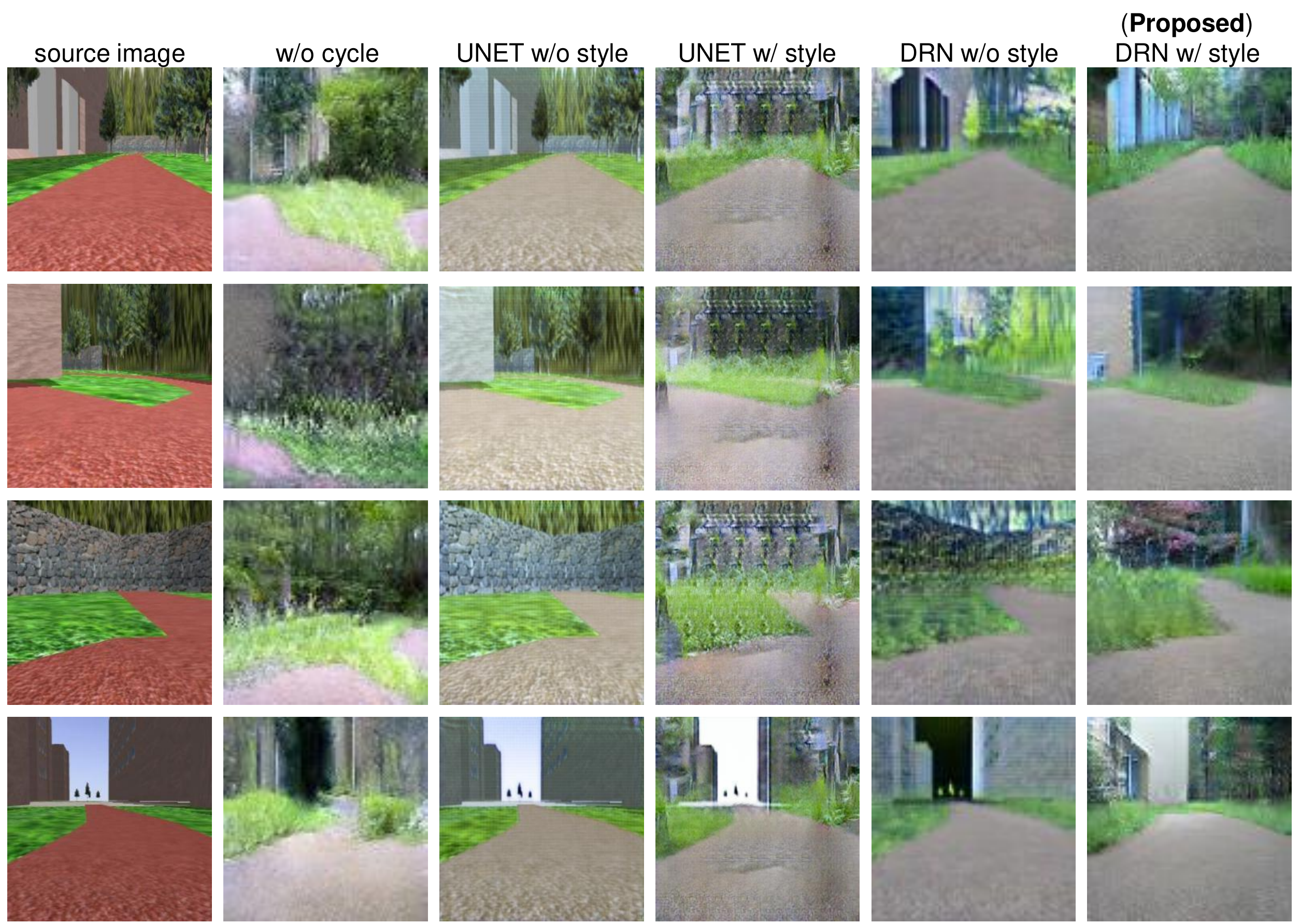}}
\caption{The source images and the transferred images by five models.}
\label{transferred_images}
\end{center}
\vskip -0.1in
\end{figure}

We compare five models categorized by: 1) with/without a cycle; (The model without a cycle transfers the source images to the target images but not the opposite.) 2) DRN/UNET in the generator; and 
3) with/without style loss. 

Fig.\ref{transferred_images} shows the source images (first column) and the corresponding transferred images obtained by five models. 1) As shown in the figure, the model without a cycle generates realistic images (second column) but fails to preserve the content of the source images. In particular, this model changed the direction and shape of the road, which is the most important part for navigation. 2) Among the models with a cycle, the models with UNET in the generator (third, fourth columns) seem to change only the color and preserve the content of the source image overly strictly compared with the models with DRN in the generator (fifth, sixth columns). The input image has a short path to the transferred image and thus to the reconstructed image in the UNET architecture. We argue that this short path may allow the model retain low level information and thus not change much from one domain to the other in order to reduce the cycle loss. In addition, the UNET with style loss model (fourth column) seems to portray textures but repeating artifacts exist. 3) Finally, the style loss appears to be critical for allowing the transferred image to have realistic textures. For example, the model with style loss successfully converted the grass plane in the source image into realistic grass with volume. In addition, the building in the simulator appears to have glass windows. More images are shown in the supplementary material S3.

We only shows one of the four possible combination for the model without a cycle, the model with DRN in the generator, style loss and without a cycle in Fig.\ref{transferred_images}. We empirically checked that the model without a cycle always fails to preserve the content. In summary, the quality of the transferred images supports our argument that using cycle loss, style loss, and DRN for the generator together is critical for the performance.

\subsection{Outdoor Navigation Test Results}
\ctable[
caption = \textbf{Average number of human interrupts to finish per a trial},
width = 1.01\textwidth,
star,
pos=t,
topcap,
label=navigation_result
]{c|cccccccc}{
}{
\toprule
\multirow{2}{*}{\small Model} &
\multicolumn{2}{c}{\small course1} &
\multicolumn{2}{c}{\small course2} &
\multicolumn{2}{c}{\small course3} \\ 
\cline{2-9}
 & \small{intervention} & \small{failed place} & \small{intervention} &\small{failed place} & \small{intervention} & \small{failed place} \\ 
\midrule
\small Source SL &\small $>$5 &\small X &\small $>$5 &\small X &\small $>$5 &\small X \\ \hline
\small UNET w/o style &\small 1.6 &\small 0.6 &\small 1.3 &\small 0.6 &\small 1.0 &\small 0.0  \\ \hline
\small UNET w/ style &\small 3.0 &\small 3.0 &\small 1.5 &\small 0.5 &\small 1.0 &\small 1.0  \\ \hline
\small DRN w/o style &\small 3.6 &\small 2.0 &\small 4.0 &\small 1.3 &\small 1.0 &\small 1.0  \\ \hline
\small\textbf{DRN w/ style} &\small 0.0 &\small 0.0 &\small 0.6 &\small 0.0 &\small 0.0 &\small 0.0  \\ \hline
}
\begin{table} [t]
\caption{\textbf{Success rate to pass the complex intersections}}\label{intersection_success_rate}
\centering
		\begin{tabular}{ccccc}\hline
			 &\small UNET w/o style &\small UNET w/ style & \small DRN w/o style & \small \textbf{DRN w/ style} \\ \hline
			\small Success rate &\small 3/5 &\small 0/5 & 2/5 & 5/5 \\ \hline
		\end{tabular}
\end{table}

In the outdoor test, we measure how well our navigational system finished various courses using a protocol similar to that in ~\cite{Rccar1}. First, we measure the number of failed places and interventions for three courses. If a vehicle goes off the road or turns against a given command, we intervene in the vehicle driving (increase the number of intervention) and then reposition the vehicle in order to try again. If the vehicle fails again, this increases the number of the failed places (but not counted as an intervention) and we reroute the vehicle manually. We also measure the rate of success in passing a complex intersection under the given command.

We tested our model on four courses. Each course is illustrated in Fig.\ref{course}. Courses 3 and 4 are unseen during the training. We do not take measurements on course 4 because this course does not contain an intersection. Instead, on course 4, we demonstrate our model's ability to navigate an unseen course as this course still requires acceptable navigational skill.

We compared the performance of five models. We exclude the model without a cycle in Sec.\ref{compare} because it changes the content and distorts the road in particular. We included the model trained with the source images and labels (supervised learning model--- named Source SL). Tab.\ref{navigation_result} shows the average number of interventions and failed places for each model while finishing the courses. 

The Source SL model failed to finish the course within reasonable trials. Given the performance of the Source SL model, we could conclude that the application of the model trained only with simulator data in outdoor environments is likely to fail. This means that domain adaptation is essential. In addition, the model with DRN in the generator, a cycle and style loss exhibited the best performance. Only this model finished all the courses without an intervention. Along with upper bound in Sec.\ref{md} and the result in the previous section, we may argue that generating a realistic image (\ie reducing the domain discrepancy) using our method is crucial. In addition, we confirmed that our proposed model navigated on course 4 well. From the results obtained for courses 3 and 4, we may argue that our model learn generality for unseen courses. Our navigation video is available at \textcolor{red}{https://www.youtube.com/watch?v=BQXtX14e-0M}.


We tested the success rate of passing a complex intersection under the given commands. We consider failure as the case in which the aforementioned human intervention is required. 
We omit the Source SL owing to its low performance as shown in Tab.\ref{navigation_result}. As shown in Tab.\ref{intersection_success_rate}, only the model with DRN in the generator, a cycle and style loss successfully navigated in all situations. 






\section{Discussion}\label{discussion}
We succeeded in performing autonomous navigation in courses with various intersections without using labels for the real images. This approach may extend the applicability of autonomous navigation to various environments. Indeed, our proposed method can be applied to other image generation or robotics tasks. 

Our style loss plays an important role in generating images with realistic texture. 
The underlying ground for using predefined networks such as VGG \cite{VGG} in conventional style transfer methods is that the networks may learn meaningful features during the training phase for the image classification. 
The gram matrix from these features can thus represent the style of the image indirectly. We may speculate that our discriminator also learns the meaningful features in order to discern the real images from synthetic images, thus resulting in the same effect. By using the discriminator instead of typically large independent networks, we were able to alleviate memory issues during the training. We also attempted to use histogram matching \cite{style4} but did not observe any positive results.

The upper bound in Sec.\ref{md} indicates that successful adversarial learning is important. We demonstrated the closeness between the two distributions of the transferred and target images by embedding images from each to a low-dimensional space, using t-SNE~\cite{tsne}. Fig.~\ref{tsne} shows the embedded points for various images.
The distribution of the transferred images is almost overlapped with the distribution of the target images while the source images are not, which supports our result. Figures obtained using other embedding methods such as principal component analysis (PCA) and, kernel PCA \cite{kpca, taehoon1} are included in the supplementary material S4.

It can be said that the task in Sec.\ref{experiment} can be achieved by filtering the red channel, learning classifier in the simulator without any domain adaptation. However, in that case, we would limit the characteristics of the target domain in advance (it is necessary to know the color of the road in the target domain). In contrast, we argue that our model need not specify the target domain by changing the road in the simulator to a gravel road which is different from pavements in the real world, and checking the quality of the domain adaptation. Fig.\ref{sandroad} shows a similar adaptation result and that our model need not specify the characteristics of the target domain.

In Sec.\ref{md}, we considered a mild assumption (\ie Assumption~\ref{assump1}). 
Thus, we considered the case in which we could eliminate the assumption. If we can collect labeled real images, an adversarial learning module 
can take not only images but also their labels as an input. We can replace all marginal distributions with the joint distribution of image $x$ and label $y$ without modifying the conclusion. 
In this case, although we use labels for the real images, we may still benefit from our approach in terms of performance. ~\citet{Berkeley} showed the improvement of the success rate in a grasping task by using domain adaptation with real labeled data.

%

\begin{figure}[t]
\captionsetup{justification=raggedright}  
\centering
  \begin{minipage}[t]{.45\textwidth}
  \centering
  \includegraphics[width=\linewidth, height=4cm]{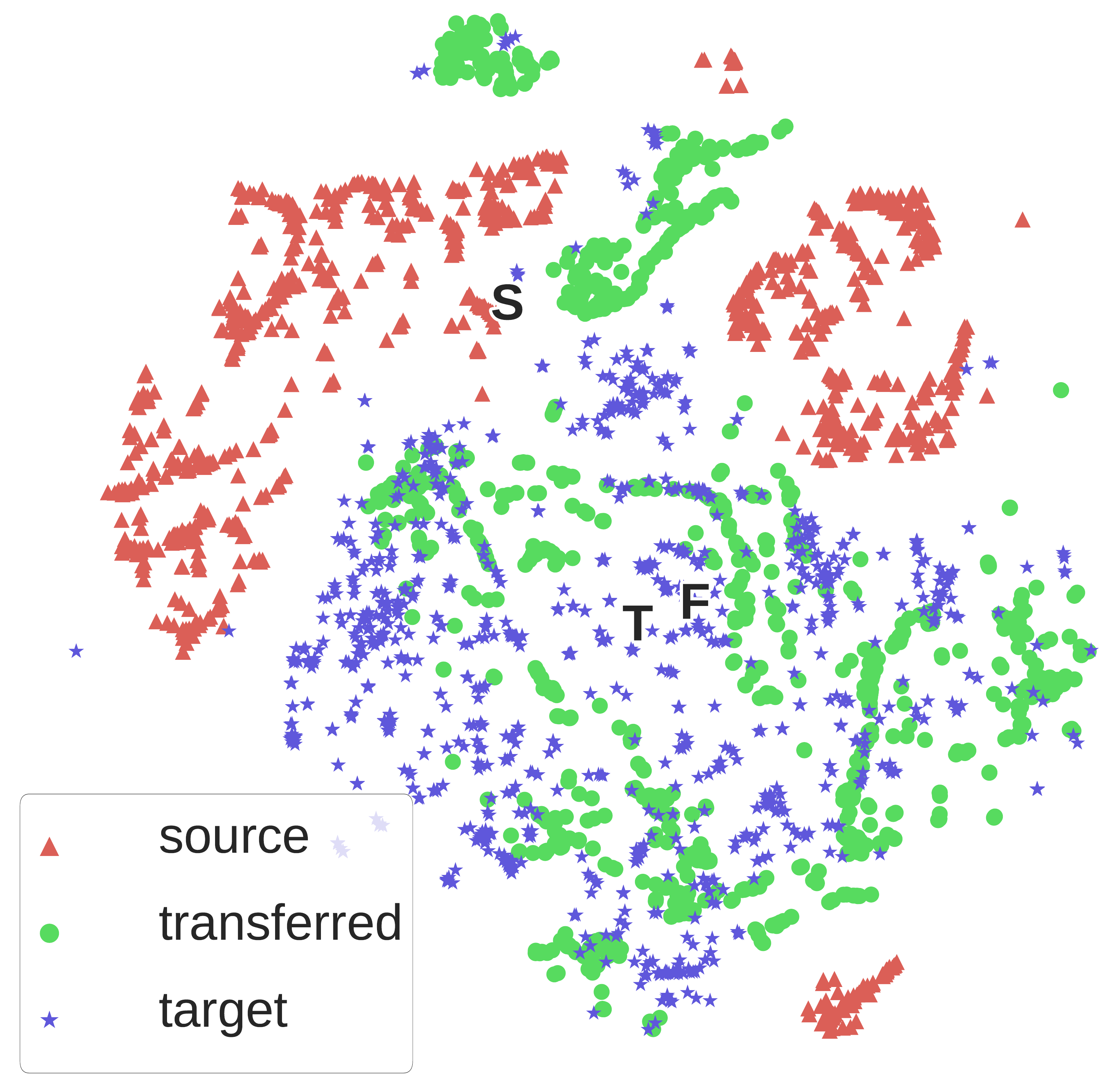}
  \caption{t-SNE result. `S',`F' and `T' indicate the median of source images, transferred images and target images.}
  \label{tsne}
  \end{minipage}
  \begin{minipage}[t]{.45\textwidth}
  \includegraphics[width=\linewidth, height=4cm]{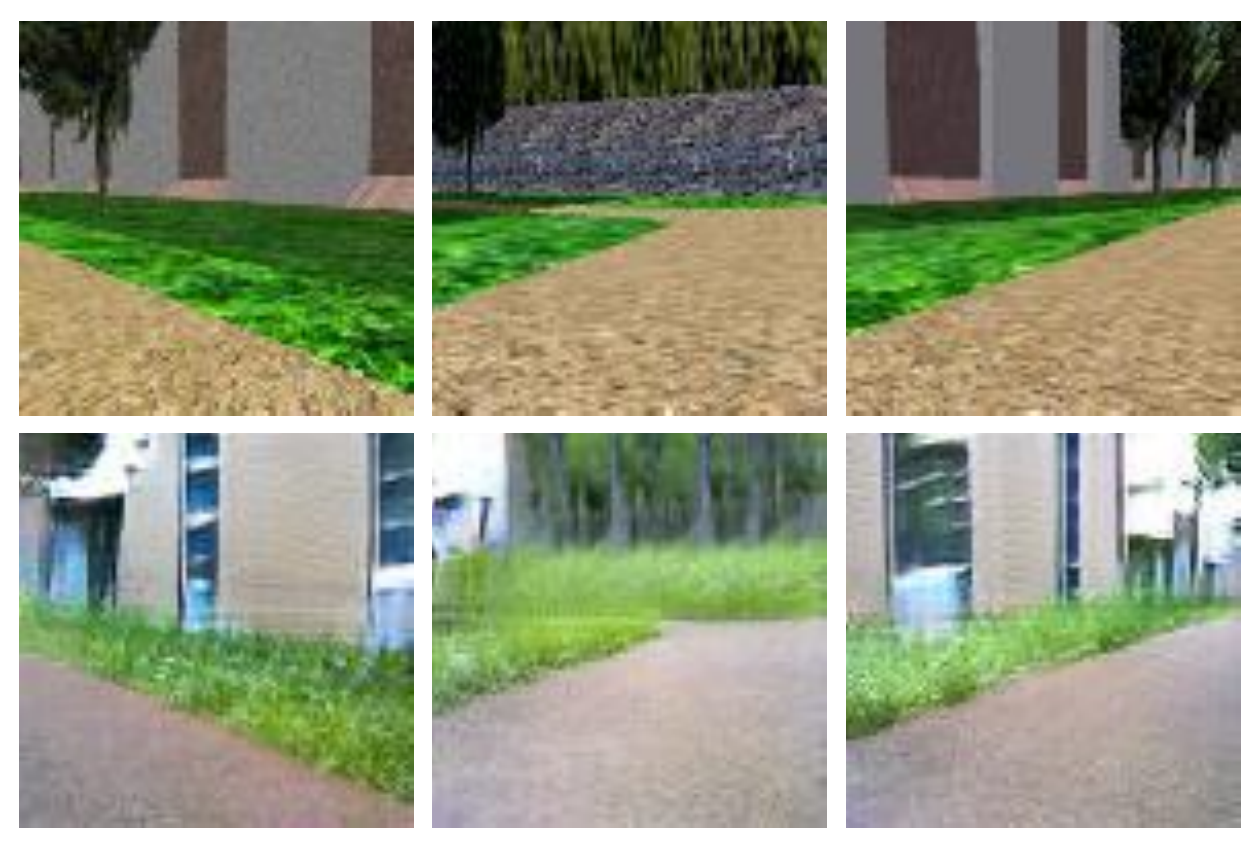}
  \caption{Transferred images (Top) from source images with yellow gravel road (Bottom).}
  \label{sandroad}
  \end{minipage}
\end{figure} 

We trained our navigational model in a reactive manner. However, it could be possible to combine reinforcement learning with our method to learn a long-term planner. The shortcoming of reinforcement learning \cite{choi2018reinforcement} --- which requires trial and error, which makes it difficult to apply in the real world ---can be addressed by using a simulator and our method. We leave this as our future work.


\bibliographystyle{plainnat}
\bibliography{main}

 \clearpage
 \pagebreak
 \begin{center}
 \textbf{\large Supplementary Material}
 \end{center}

 \beginsupplement
 \section{Implementation details}\label{implementation details}

We use the Adam optimizer \cite{adam} to minimize the objective function with a learning rate of 0.0001. We do not decay a learning rate. All the weights are randomly initialized from a truncated normal distribution with a standard deviation of 0.02, and we use a batch size of 8 for the DRN generator and 4 for the UNET generator. We use a data augmentation technique such as image contrast, brightness, saturation and injecting Gaussian noise with a probability 0.5 to diverse sample distribution. In the GAN training, the discriminator often overpowers the generator, thus hampering the generator from learning effectively. We add additive Gaussian noise and a dropout \cite{srivastava2014dropout} to the hidden layers of the discriminator to balance the power between the generator and discriminator.

We adapt our architecture from \cite{isola2017image,unsupervised}. Tab.\ref{gen} shows the architecture details of each component of our model. We denote the filter size, stride, number of the feature map and dilation rate as $\textbf{F}$, $\textbf{S}$, $\textbf{C}$, and $\textbf{D}$, respectively. For example, $\textbf{F3S1C64}$ denotes a convolution layer wherein 3$\times 3$ spatial filters are applied with a stride of 1, generating 64 feature maps. If $\textbf{D}$ is appended, the dilation convolution is applied with a corresponding dilation rate. In Tabs.\ref{drng} and \ref{classifier}, layers with two multi rows involve a residual connection. In Tab.\ref{unetg}, the arrow represents concatenation to the layers of the generator that are indicated. In Tab.\ref{disc}, $Lrelu$ denotes a leaky relu function with a slope of 0.2. It should be noted that we do not apply non-linear activations to the last layer of each component. 

We implemented the classifier as a branched architecture in which each branch shares the classifier until the $14_{th}$ layer as shown in Tab.~\ref{classifier}. We have five output nodes for each command that produce probabilities for five steering command intervals.

We apply pixel normalization \cite{karras2017progressive} for the generator, which normalizes the feature vector to the unit length so that local response of the feature map is well maintained. We empirically verify that the generator with pixel normalization produces a sharper result than other normalization techniques such as instance normalization \cite{style1,huang2017arbitrary} and batch normalization \cite{ioffe2015batch}. 

We apply instance normalization to the discriminator as in the case of other style transfer and domain adaptation approaches \cite{CYCADA,shu2018dirt,huang2017arbitrary, cyclegan,style1}. We use a small batch size owing to the memory capacity, and the performance of the batch normalization degenerates as the batch size decreases, as the small number of samples in a batch does not represent the entire dataset \cite{wu2018group}. We, thus, adopt group normalization \cite{wu2018group} for the classifier, which normalizes the input tensor with respect to a group of channels in order to guarantee the classification performance even with a small batch size. It should be noted that we do not apply normalization to the last layer of each component.

We send an angular velocity command to steer the robot at 4Hz. We divide the angular velocity in [-1.4,1.4]rad/s into five intervals and set a middle value of the range as the representative value for each class. During the navigational test in the real world, we send a weighted average angular velocity calculated with the representative values and our model's predicted class probability as the weight. 

\clearpage

\begin{table}[h]\caption{Architecture details of the generator, discriminator and the classifier.}
     \begin{subtable}{.5\linewidth}
       \centering
         \caption{DRN generator}
         \begin{tabular}{c|c}
           Layer Index & {DRN} \\ \hline
           Input & 96$\times$96$\times$3 Image \\
           \hline
 		  	L - 1 & Conv, F7S1C64, Relu \\\hline
           	\multirow{2}{*}{L - 2} & Conv, F3S1C64, Relu \\
           	 & Conv, F3S1C64, Relu \\ \hline
           	L - 3 & 2$\times$2 max-pool, S2 \\ \hline
           	\multirow{2}{*}{L - 4} & Conv, F3S1C128, Relu \\
             & Conv, F3S1C128, Relu \\ \hline
           	\multirow{2}{*}{L - 5} & Conv, F3S1C128, Relu \\
           	 & Conv, F3S1C128, Relu \\ \hline
            \multirow{2}{*}{L - 6} & Conv, F3S1C256D1, Relu \\
             & Conv, F3S1C256D2, Relu \\ \hline
            L - 7 & Conv, F3S1C128, Relu \\ \hline
            L - 8 & Conv, F3S1C128, Relu \\ \hline
            L - 9 & Deconv, F3S2S64, Relu \\ \hline
            L - 10 & Deconv, F7S1S3 
         \end{tabular} \label{drng}
     \end{subtable}  
     \begin{subtable}{.5\linewidth}
       \centering
         \caption{UNET generator}
         \begin{tabular}{c|c}
           Layer Index & {UNET} \\ \hline
           Input & 256$\times$256$\times$3 Image \\
           \hline
 		  	L - 1 & Conv, F4S2C64, Relu \tikzmark{a1} \\ \hline
           	L - 2 & Conv, F4S2C128, Relu \tikzmark{a2}\\ \hline
           	L - 3 & Conv, F4S2C256, Relu \tikzmark{a3}\\ \hline
           	L - 5 & Conv, F4S2C512, Relu \tikzmark{a4}\\ \hline
            L - 6 & Conv, F4S2C512, Relu \tikzmark{a5}\\ \hline
           	L - 7 & Conv, F4S2C512, Relu \tikzmark{a6} \\ \hline
           	L - 8 & Conv, F4S2C512, Relu \tikzmark{a7} \\ \hline
            L - 9 & Conv, F4S2C512, Relu \\ \hline
            L - 10 & Deconv, F4S2C512, Relu \tikzmark{b1}\\ \hline
            L - 11 & Deconv, F4S2C512, Relu \tikzmark{b2}\\ \hline
            L - 12 & Deconv, F4S2C512, Relu \tikzmark{b3}\\ \hline
            L - 13 & Deconv, F4S2C512, Relu \tikzmark{b4}\\ \hline
            L - 14 & Deconv, F4S2S256, Relu \tikzmark{b5}\\ \hline
            L - 15 & Deconv, F4S2C128, Relu \tikzmark{b6}\\ \hline
            L - 16 & Deconv, F4S2C128, Relu \tikzmark{b7}\\ \hline
            L - 17 & Deconv, F4S2C3
         \end{tabular} \label{unetg}
         \begin{tikzpicture}[overlay, remember picture, yshift=.25\baselineskip, shorten >=.5pt, shorten <=.5pt]
    \draw [->] ({pic cs:a7}) [bend left] to ({pic cs:b1});
    \draw [->] ({pic cs:a6}) [bend left] to ({pic cs:b2});
    \draw [->] ({pic cs:a5}) [bend left] to ({pic cs:b3});
    \draw [->] ({pic cs:a4}) [bend left] to ({pic cs:b4});
    \draw [->] ({pic cs:a3}) [bend left] to ({pic cs:b5});
    \draw [->] ({pic cs:a2}) [bend left] to ({pic cs:b6});
    \draw [->] ({pic cs:a1}) [bend left] to ({pic cs:b7});
  		 \end{tikzpicture}
     \end{subtable} 
          \begin{subtable}{.5\linewidth}
       \centering
         \caption{Discriminator}
         \begin{tabular}{c|c}
           Layer Index & {Discriminator} \\ \hline
           Input & 96$\times$96$\times$3 Image \\
           \hline
 		  	\multirow{3}{*}{L - 1} & Conv, F4S1C64, Lrelu \\
           	 & Dropout, $\rho=0.5$ \\
           	 & Gaussian noise, $\sigma=0.2$ \\ \hline
           	\multirow{3}{*}{L - 2} & Conv, F4S2C128, Lrelu \\
             & Dropout, $\rho=0.5$ \\ 
           	 & Gaussian noise, $\sigma=0.2$ \\ \hline
           	\multirow{3}{*}{L - 3} & Conv, F4S2C256, Lrelu \\ 
             & Dropout, $\rho=0.5$ \\
             & Gaussian noise, $\sigma=0.2$ \\ \hline
            \multirow{3}{*}{L - 4} & Conv, F4S2C512, Lelu \\ 
             & Dropout, $\rho=0.5$ \\ 
             & Gaussian noise, $\sigma=0.2$ \\ \hline
            L - 14 & Conv, F4S2C1024, Lrelu \\ \hline
            L - 15 & Conv, F1S1C1  
         \end{tabular}\label{disc}
     \end{subtable}  
          \begin{subtable}{.5\linewidth}
       \centering
         \caption{Classifier}
         \begin{tabular}{c|c}
           Layer Index & {Classifier} \\ \hline
           Input & 96$\times$96$\times$3 Image \\
           \hline
 		  	L - 1 & Conv, F7S2C64, Relu \\\hline
           	\multirow{2}{*}{L - 2} & Conv, F3S1C64, Relu \\
           	 & Conv, F3S1C64, Relu \\ \hline
            \multirow{2}{*}{L - 3} & Conv, F3S1C64, Relu \\
           	 & Conv, F3S1C64, Relu \\ \hline
           	L - 6 & 2$\times$2 max-pool, S2 \\ \hline
           	\multirow{2}{*}{L - 4} & Conv, F3S1C128, Relu \\
             & Conv, F3S1C128, Relu \\ \hline
           	\multirow{2}{*}{L - 5} & Conv, F3S1C128, Relu \\
           	 & Conv, F3S1C128, Relu \\ \hline
            \multirow{2}{*}{L - 6} & Conv, F3S1C256D1, Relu \\
             & Conv, F3S1C256D2, Relu \\ \hline
            \multirow{2}{*}{L - 7} & Conv, F3S1C256D2, Relu \\
             & Conv, F3S1C256D2, Relu \\ \hline
            \multirow{2}{*}{L - 8} & Conv, F3S1C512D2, Relu \\
             & Conv, F3S1C512D4, Relu \\ \hline
            \multirow{2}{*}{L - 9} & Conv, F3S1C512D4, Relu \\
             & Conv, F3S1C512D4, Relu \\ \hline
            L - 10 & Conv, F3S1C512D2, Relu \\ \hline
            L - 11 & Conv, F3S1C512D2, Relu \\ \hline
            L - 12 & Conv, F3S1C512, Relu \\ \hline
            L - 13 & Conv, F3S1C512, Relu \\ \hline
            L - 14 & Global average pooling \\ \hline
             & Branches for command\\ \hline
            L - 15 & FC, C5, Softmax 
         \end{tabular} \label{classifier}
     \end{subtable} 
 \label{gen}
\end{table}

\clearpage

\section{Proof of Theorem 1 in Sec.3.3}\label{proof of thm}
\begin{thm}
$E_\mathrm{T} \leq 4 \sqrt{X_2^2(P_\mathrm{T}(x),P_\mathrm{F}(x))} + E_\mathrm{F}$ where $X_2^2$ is the Chi-squared distance.
\end{thm}

Before we prove the theorem, we go through following two propositions.

\begin{prop}\label{prop1}
For $P_T(x)$ and $\frac{P_T(x)+P_F(x)}{2}$, the following inequality holds:
\begin{equation} \label{eq:tve}
\begin{aligned}
\scriptstyle{TV(P_{T}(x)||\frac{P_{T}(x)+P_{F}(x)}{2})^2 \leq \frac{1}{4}X_2^P(P_{T}(x),\frac{P_{T}(x)+P_{F}(x)}{2})} 
\end{aligned}
\end{equation}
where $TV$ is the total variation distance on the probability space defined by $TV(\mathbb{P}||\mathbb{Q}) = \frac{1}{2}\int |\mathbb{P}-\mathbb{Q}| d\mu$ for given measure $\mu$ and two probability distributions, $\mathbb{P}$ and $\mathbb{Q}$. $X_2^P$ is the Pearson divergence between $\mathbb{P}$ and $\mathbb{Q}$.
\end{prop}

\begin{proof}
\citet{choosing} showed that $TV(\mathbb{P}||\mathbb{Q})^2 \leq \frac{1}{4}X_2^P(\mathbb{P},\mathbb{Q})$ holds if $\mathbb{P}$ is dominated by $\mathbb{Q}$ where $\mathbb{P}$ and $\mathbb{Q}$ are probability distributions.

Because the support of $P_\mathrm{T}(x)$ is contained in the support of $\frac{P_\mathrm{T}(x)+P_\mathrm{F}(x)}{2}$, we can derive Equation~\ref{eq:tve} by substituting $P_{T}(x)$ for $\mathbb{P}$ and $\frac{P_\mathrm{T}(x)+P_\mathrm{F}(x)}{2}$ for $\mathbb{Q}$.
\end{proof}

\begin{prop}\label{prop2}
\begin{equation}\label{eq:tvep}
TV(P_T(x)||P_F(x))<2 \sqrt{X_2^2(P_{T}(x),P_{F}(x))}.
\end{equation}
\end{prop}

\begin{proof}

the relationship between $X_2^2$ and the Pearson divergence, $X_2^P$, such that $X_2^2(\mathbb{P},\mathbb{Q})=\frac{1}{4}X_2^P(\mathbb{P},\frac{\mathbb{P}+\mathbb{Q}}{2})$ holds \cite{FISHERGAN}, gives us the following relation:



\begin{align} \label{eq:finaltv}
\scriptstyle TV(P_{T}(x)||P_{F}(x)) &\scriptstyle= \frac{1}{2}\int |P_{T}(x)-P_{F}(x)| dx \\
&\scriptstyle= \int |P_{T}(x)- \frac{P_{T}(x)+P_{F}(x)}{2}|dx \\
&\scriptstyle=2TV(P_{T}||\frac{P_{T}+P_{F}}{2}) \\
&\scriptstyle \,<\, 2\sqrt{\frac{1}{4}X_2^P(P_{T},\frac{P_{T}+P_{F}}{2})} \label{ineqlast} \\ 
&\scriptstyle=2\sqrt{X_2^2(P_{T},P_{F})}. 
\end{align}
The inequality of Eq. \ref{ineqlast} comes from Prop.~\ref{prop1}.
\end{proof}

We now show the theorem using Prop.~\ref{prop2} as follows:
\begin{align}
\scriptstyle E_\mathrm{T} &\scriptstyle= \int_{\mathcal{X} \times \mathcal{Y}} e(x,y)P_\mathrm{T}(x,y)dxdy \label{eq:start} \\
&\scriptstyle= \int_{\mathcal{X} \times \mathcal{Y}} e(x,y)(P_\mathrm{T}(x,y)-P_\mathrm{F}(x,y)+P_\mathrm{F}(x,y))dxdy \\
&\scriptstyle \leq \int_{\mathcal{X} \times \mathcal{Y}}
e(x,y)|P_\mathrm{T}(x,y)-P_\mathrm{F}(x,y)|dxdy + E_\mathrm{F} \\
&\scriptstyle \leq \int_{\mathcal{X} \times \mathcal{Y}}
|P_\mathrm{T}(x)P_\mathrm{T}(y|x)-P_\mathrm{F}(x)P_
\mathrm{F}(y|x)\big|dxdy + E_\mathrm{F} \label{thm2}\\
&\scriptstyle \approx  \int_\mathcal{X}
|P_\mathrm{T}(x)-P_\mathrm{F}(x)|dx\int_\mathcal{Y} P_\mathrm{T}(y|x)dy + E_\mathrm{F} \label{thm3}\\ 
&\scriptstyle = 2TV(P_\mathrm{T}||P_\mathrm{F}) + E_\mathrm{F}\, <\, 4\sqrt{X_2^2(P_{T},P_{F})} + E_\mathrm{F} \label{eq:end}
\end{align}
where we apply $e(x,y)\leq 1$ to Eq.~\ref{thm2}, and Assumption 1 to Eq. \ref{thm3}.

\section{Additional Transferred Images}

We show additional transferred images in Fig.\ref{supple_tr} in addition to Fig.4. Our proposed model generates the most realistic images while preserving the content of the source image.

\clearpage 

\begin{figure*}[h]
\begin{center}
\centerline{\includegraphics[width=\textwidth,height=\textheight,keepaspectratio]{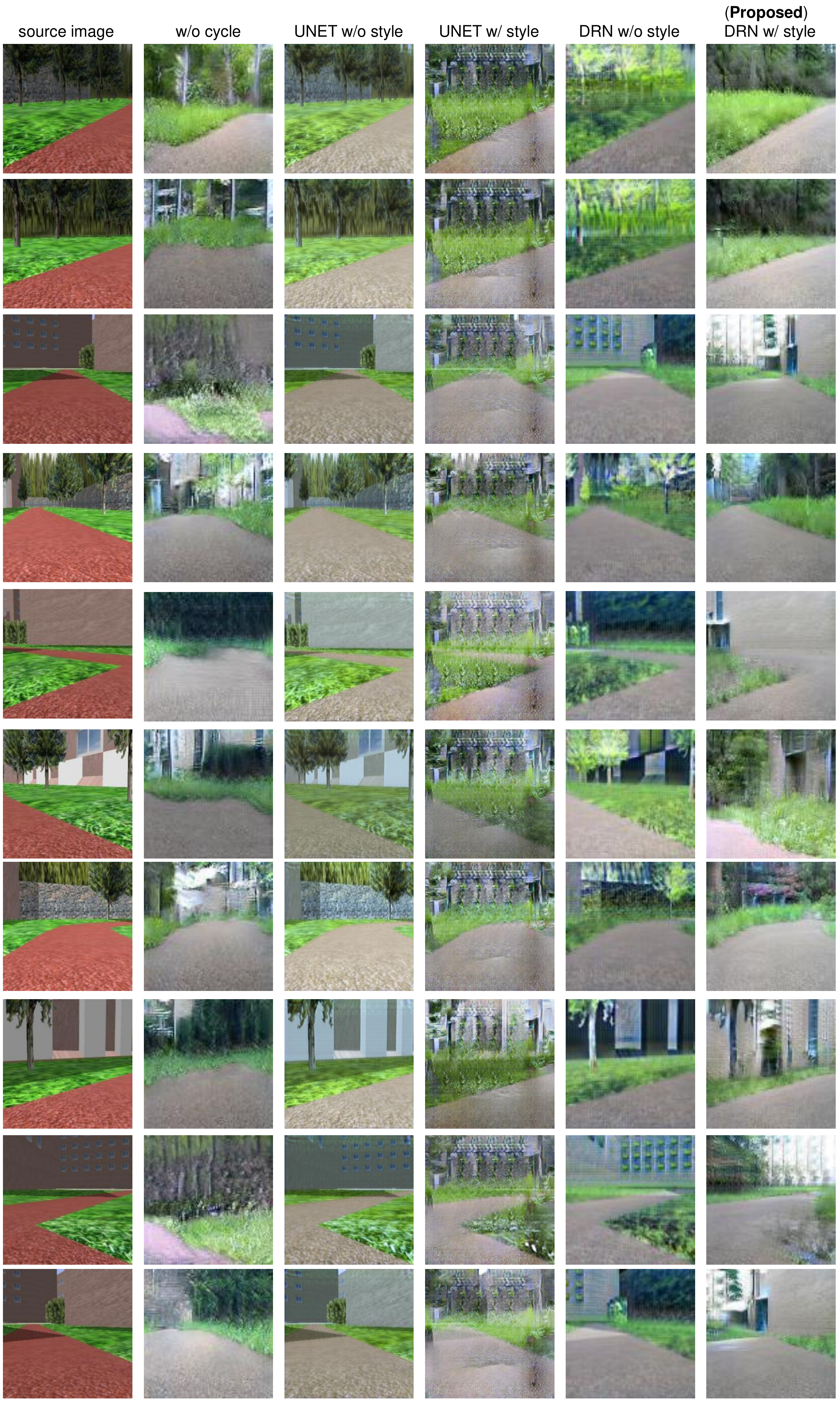}}
\caption{Source images and transferred images from five models.}
\label{supple_tr}
\end{center}
\end{figure*}

\clearpage
\section{Embedding by Other Methods}
We used other embedding methods to determine if the transferred images are actually similar to the target images, in addition to t-SNE in Sec.5. We used isomap \cite{isomap}, principal component analysis (PCA) \cite{pca} and kernel PCA \cite{kpca}. Fig.\ref{PCAcosine}, \ref{PCA}, and \ref{isomap} represent the embedding results obtained using PCA with cosine kernel, PCA, and isomap respectively. The transferred images moved toward the target images in all figures. $\mathtt{S, T}$, and $\mathtt{F}$ in each figure, indicate the median of the source, transferred, and target domain embedding images, respectively.

\begin{figure}[!htb]
\minipage{0.5\textwidth}
  \includegraphics[width=\linewidth]{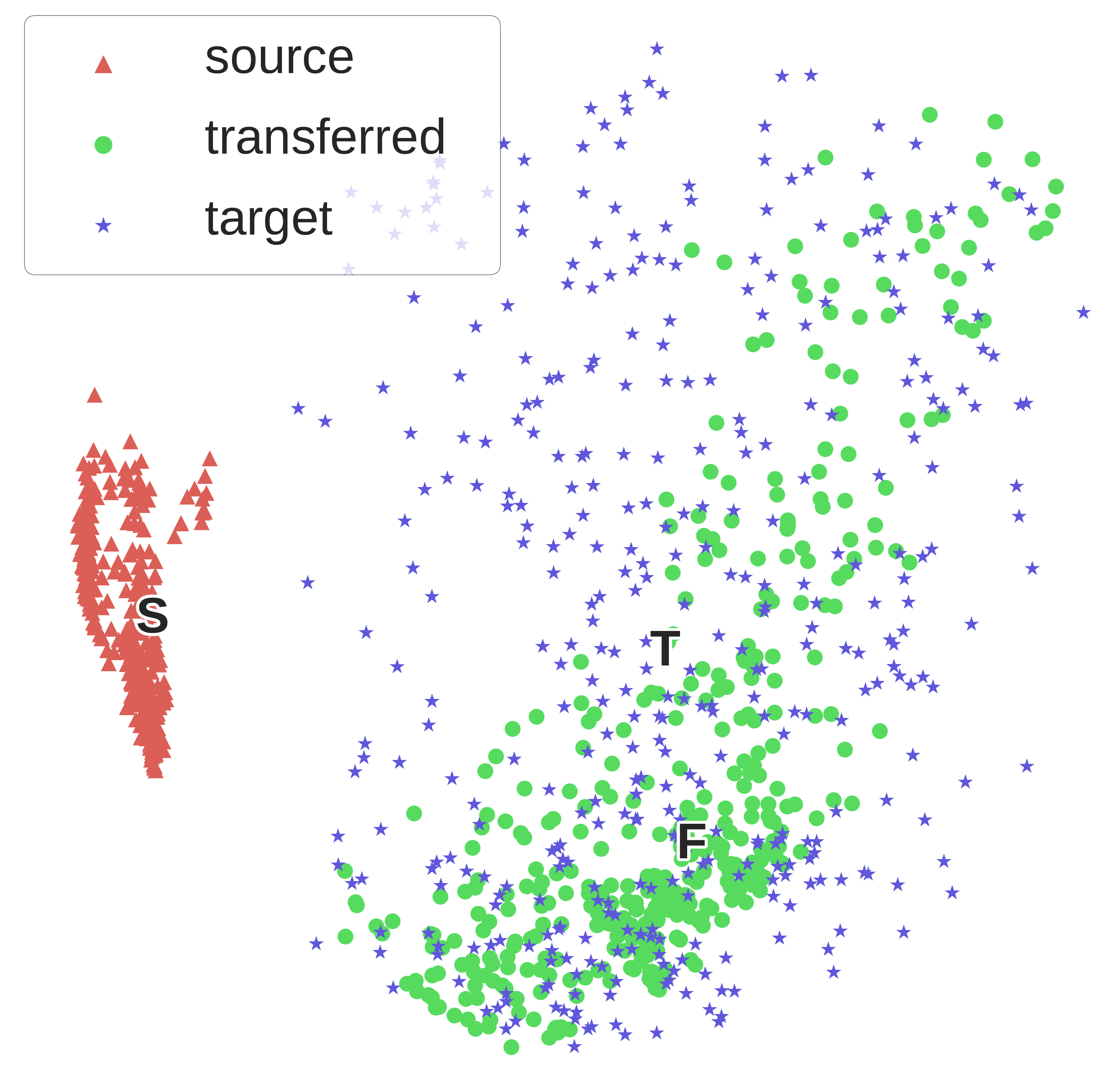}
  \caption{Cosine kernel PCA result.}\label{PCAcosine}
\endminipage\hfill
\minipage{0.5\textwidth}
  \includegraphics[width=\linewidth]{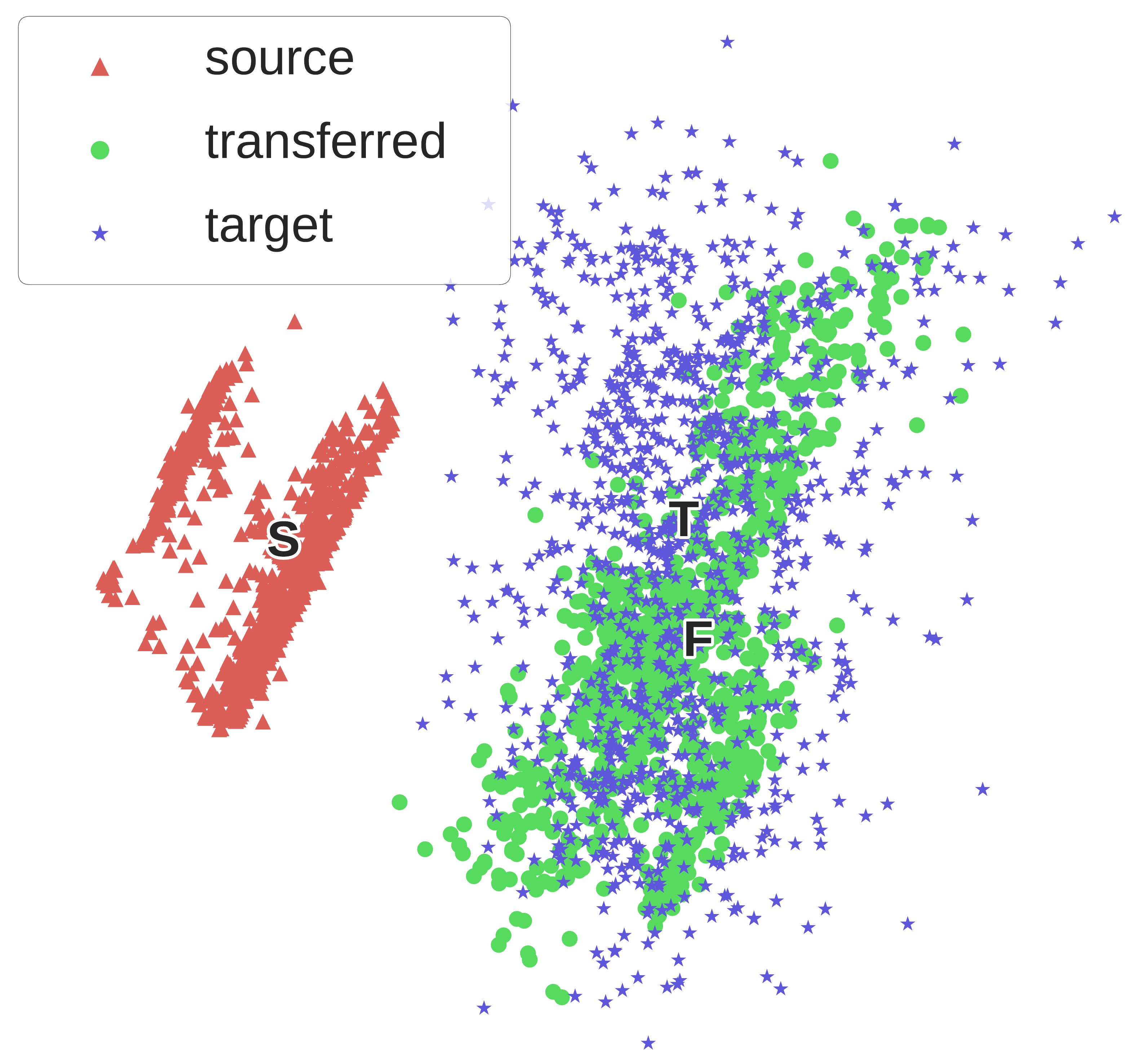}
  \caption{Standard PCA result.}\label{PCA}
\endminipage\hfill \\
  \centering
  \includegraphics[width=0.5\textwidth]{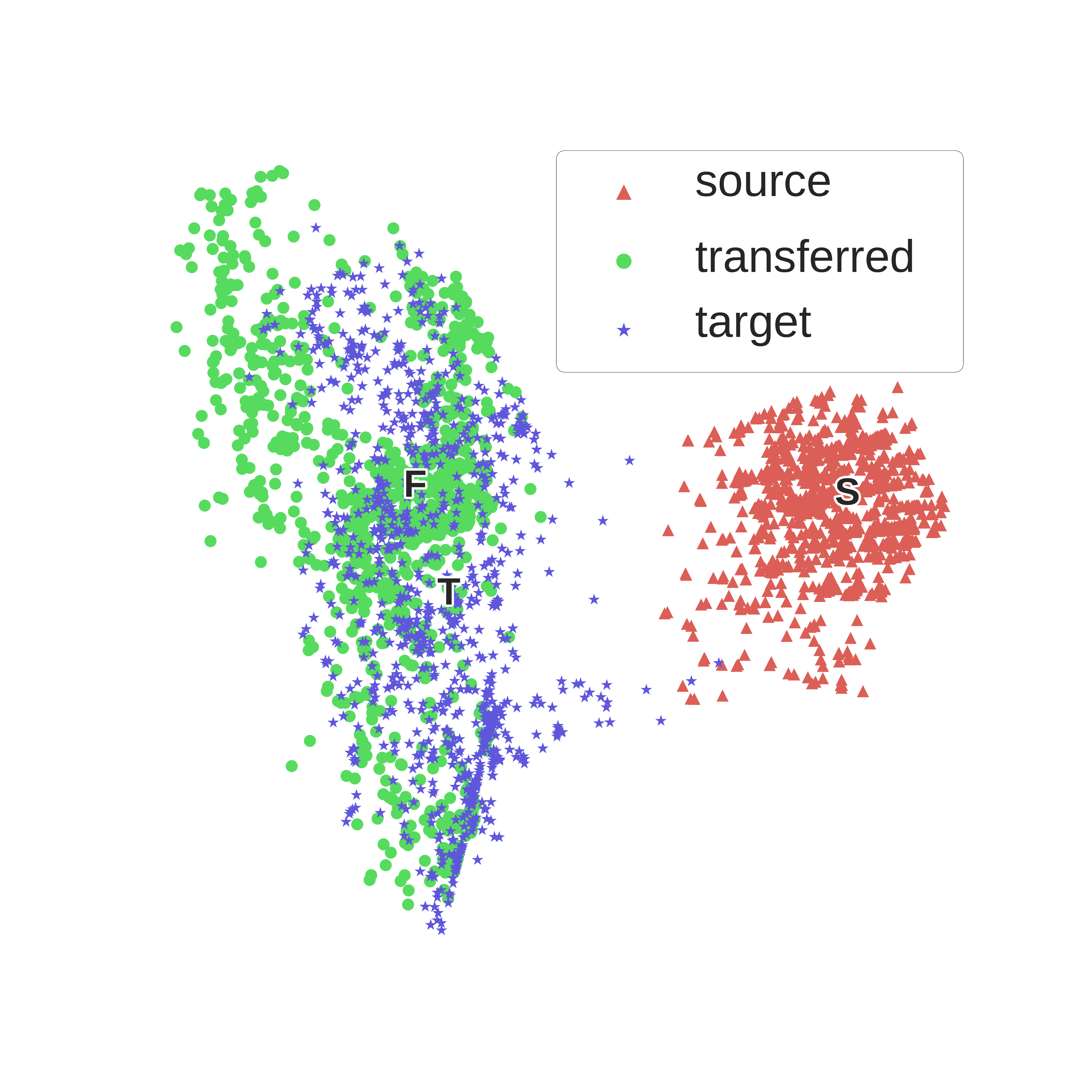}
  \caption{Isomap result.}\label{isomap}
\end{figure}

\clearpage

\section{Additional Experiment}
In addition to our navigation experiment, we performed an additional task to classify the direction in which the camera is heading with respect to the road in the local trail environment. It could be stated that this task has a strong relation with the road following navigation task because the model that learns for the task could perform UAV road following as shown in \cite{NVIDIA,guisti}.

\begin{figure}[h]
\begin{center}
\centerline{\includegraphics[width=\columnwidth]{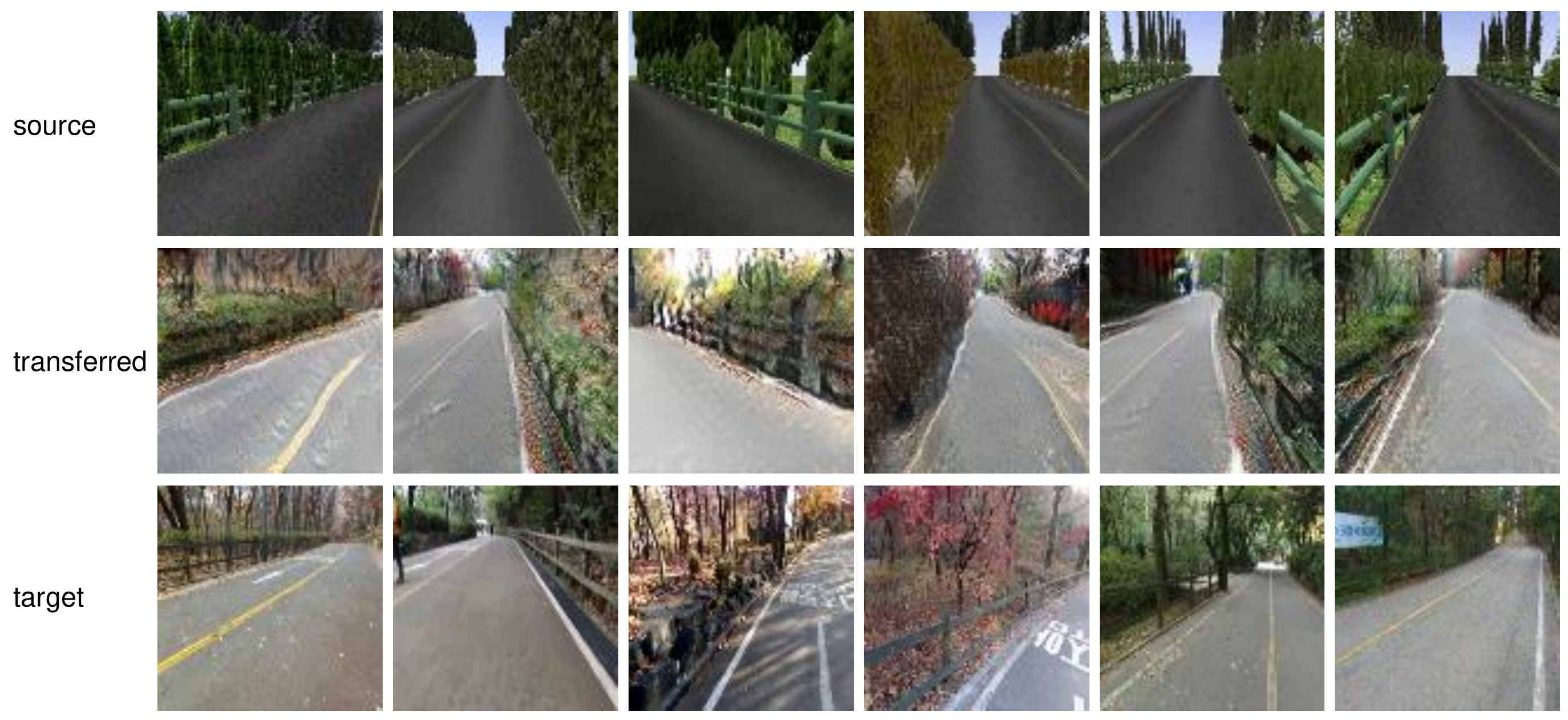}}
\caption{Source images, transferred images and target images for local trail.}
\label{supple_gwanak}
\end{center}
\end{figure}

\ctable[
caption = Classification accuracy for three models,
center,
star,
pos=h,
topcap,
label=accuracy
]{c|ccc}{
}{
\toprule
\multirow{2}{*}{Model} &
\multicolumn{3}{c}{Ground truth} \\
\cline{2-4}
& {left} & {center} & {right} \\ 
\midrule
Source SL & 87.9\%& 66.3\% & 82.7\%  \\ 
Target SL & 99.2\% & 96.0\% & 99.5\% \\ 
\textbf{Our model} & 91.3\% & 94.2\% & 86.7\%  \\ \hline
}

\ctable[
caption=Confusion matrix for three models,
star,
pos=h,
topcap,
label=cf_mtx
]{cc|ccc|ccc|ccc}{
}{
\toprule
& & \multicolumn{3}{c}{Source SL output} & \multicolumn{3}{c}{Target SL output} & \multicolumn{3}{c}{\textbf{Our model output}} \\
\cline{3-11}
& & left & straight & right & left & straight & right & left & straight & right \\
\multirow{3}{*}{\rotatebox[origin=c]{90}{\parbox[t]{0.8cm}{Ground\\ truth}}} & left & 87.9\%& 8.3\% & 3.8\% & 99.2\% & 0.45\% & 0.33\% & 91.3\% & 8.2\% & 0.5\%\\
& straight & 15.1\%& 66.3\%& 18.6\% & 0.32\% & 96.0\% & 3.7\% & 3.0\% & 94.2\% &2.8\%\\
& right & 13.0\% & 4.3\%& 82.7\% &0.03\% & 0.48\% & 99.5\% & 2.1\% & 11.2\% & 86.7\% \\ \hline
}

We set target domain images as local trail images. We then developed a simulator that contains road and low-poly trees, fences, and rocks as target images. Fig.\ref{supple_gwanak} shows the source images and transferred images obtained using our model with DRN in the generator, a cycle and style loss, and the target images. The fences and bushes became more realistic in the transferred images. This shows that our proposed model is capable of generating realistic images in different environments from the one described in the body of the paper even with a low-poly simulator.

We trained three models that classify where the camera is heading for given trail images in three classes (left/straight/right with respect to the road). The first is Source SL (Supervised Learning) model that trains only using source images and corresponding labels in a supervised manner. The second is Target SL model that uses target images and their correct labels during the training. The third is our proposed model that described in the body of the paper. 

Tab.\ref{accuracy}\ref{cf_mtx} shows the classifying accuracy and its confusion matrix. Our model improved the accuracy as compared with the Source SL model. In addition, the dangerous mistake that switches left class to right class and the opposite decreased in our model compared with Source SL model. These results support that our domain adaptation method is helpful in learning a navigational system using a simulator and makes us anticipate that our model may perform lane following task well in the real world despite the gap between the accuracy of our model and that of the Target SL.

\end{document}